\documentclass[lettersize,journal]{IEEEtran}
\usepackage{amsmath,amsfonts}
\usepackage{algorithmic}
\usepackage{algorithm}
\usepackage{array}
\usepackage[caption=false,font=normalsize,labelfont=sf,textfont=sf]{subfig}
\usepackage{textcomp}
\usepackage{stfloats}
\usepackage{amsthm}
\usepackage{url}
\usepackage{verbatim}
\usepackage{mathtools}
\usepackage{amsmath}
\usepackage{amssymb}
\usepackage{graphicx}
\usepackage{cite}
\usepackage{multirow}
\usepackage{booktabs}
\usepackage{subcaption}
\newtheorem{theorem}{Theorem}[section]
\newtheorem{corollary}{Corollary}[theorem]
\newtheorem{lemma}[theorem]{Lemma}
\newtheorem{proposition}{Proposition}

\usepackage{xcolor}

\newcommand{\nf}[1]{\textcolor{black}{#1}}

\usepackage[x11names, svgnames, dvipsnames]{xcolor}
\usepackage{color}
\usepackage{colortbl}

\begin{document}

\title{A Reachability-based Safety Certificate \\ for Dynamical System Motion Policies}

\author{Aditya Vats$^{*}$, Tianyi Xia$^{*}$ and Nadia Figueroa%
        \thanks{$^{*}$These authors contributed equally to this work.}%
        \thanks{The authors are with the University of Pennsylvania,
        Philadelphia, PA, USA.}}


\maketitle

\begin{abstract}

\nf{Dynamical Systems (DS) are reactive motion policies representing vector fields trained with theoretical guarantees of stability and convergence. To ensure safety during deployment in unknown environments they must be locally reshaped, either through modulation or geometric control barrier function strategies. However, depending on the geometry of the obstacles and the complexity of the DS, these local strategies can lead the system to unavoidable collisions or spurious attractors.} In this work, we certify safety with a value function drawn from the notion of backward reachability tube, which measures the worst-case safety along a rollout trajectory \nf{of the nominal DS.}
Usually, such a value function is intractable for a controlled system due to curse of dimensionality. We show that in the \nf{DS-based} learning-from-demonstration setting, the absence of a control input collapses the reachability problem to a deterministic rollout, and the presence of certain stability conditions truncates the infinite horizon to a finite one, resulting in a well-defined value function. We further show that the value function we devised is the maximal forward-invariant subset of the obstacle-free region for the nominal DS flow. The application of this certificate function is validated across five DS constructions - analytical, Neural ODE, diffeomorphic latent space, LPV-DS, SE(3) and validate it on a Franka manipulator. Modulation and geometric CBFs also suffer from saddle point in cases of head-on approach towards an unsafe zone. We show that CBF-on-V avoids this pitfall entirely. 

\end{abstract}

\begin{IEEEkeywords}
Dynamical Systems, Backward Reachability, Control Barrier Functions, Safety.
\end{IEEEkeywords}

\section{Introduction}

\IEEEPARstart{M}{any} applications in robotics require learning a task from a few demonstrations. This problem, referred to as imitation learning, is a cornerstone in robotics. It is especially important for human-robot interaction where a trajectory may be demonstrated by a human that is to be mimicked while extending beyond the demonstration domain. Dynamical systems (DS) are first-order autonomous systems that define a vector field on a domain. A trajectory integrated from these velocity vectors is meant to imitate the demonstrations. An attractive property of a DS is its robustness to perturbation where it can adapt its trajectory from a disturbed position. With guarantees of stability imbibed in its construction, the motion plan from such a vector field converges to an equilibrium position, usually the target point at the end of a task.  

A DS, however, encodes only the demonstrated task and has no native handling of obstacles encountered at deployment. Existing approaches augment the DS with a safety mechanism, and fall into two groups. The first deflects the vector field geometrically via a multiplicative modulation matrix that reshapes the flow around an obstacle as a function of proximity to it \cite{khansari2012modulation, Huber2019-ur, onmanifold}. The second enforces safety as an additive modification to the DS where a control barrier function (CBF) to find an appropriate modification\cite{Schonger_2024, Binny_2026, 10611584}. Both groups of work are \textbf{local}, wherein the correction applied at a state depends only on the obstacle geometry at that state mainly the current margin and the current boundary normal. Thus, in complex obstacle geometries and/or with highly nonlinear DS, these local methods may induce saddle points, spurious attractors or even limit cycles, impeding the liveness \cite{10665911} of the execution.

\nf{In this work, we instead propose a certificate that is} a function of the closed-loop policy that elevates safety from a point-wise geometric property to a path property. A filter derived from such a certificate intervenes only when the uncorrected trajectory enters the obstacle, rather than whenever it is close to the obstacle. \nf{We adopt} a Hamilton-Jacobi (HJ) reachability value function to encode this global approach to the safety problem. It encodes the worst-case safety violation over the entire trajectory. It's super-level set is the maximal control invariant set to maintain forward invariance \cite{choi2021robustcontrolbarriervaluefunctions}. 

For a controlled system, computing this value function requires solving a Hamilton Jacobi PDE over the state space. The setting of this problem suffers from the curse of dimensionality where the problem becomes intractable as the dimension scales which is one of the reasons why it is not prominent as a certificate. Learning-based methods such as DeepReach \cite{bansal2020deepreachdeeplearningapproach} instead learn a neural network using a loss function that is derived the Hamilton Jacobi PDE. The learning landscape of this method is complex, involving multi-stage training. Other methods that use the reachability value function include disturbances and control inputs in their formulations. They do not exploit the structure of an autonomous motion plan. 

\nf{Interestingly, the DS-based motion planning and learning-from-demonstration (LfD) paradigm \cite{DSbook} affords} two simplifications that make a BRT (backward-reachable-tube)-like value function far more tractable to compute. First, the DS is \emph{autonomous}, i.e., no control input collapses the value function from an optimal control problem to a deterministic minimum of safety function along a rollout. Second, an observable finite convergence time lets us equate finite-time forward invariance with forward invariance for all time. With these simplifications, the reachability-based value function admits closed-form expressions for analytical DS and a learning-based construction for highly nonlinear DS. We use this value function as a safety constraint that injects a virtual control input into the DS, filtering the nominal policy without redesigning it. The construction is agnostic to how the DS is represented: analytical, Neural ODE, diffeomorphic latent-space, stochastic, and SE(3) formulations are handled by the same filter, with only the value-function computation changing from one to the next. 

To the best of our knowledge, this is the first work to apply reachability-based safety certification to learned DS motion policies. Specifically, we make
the following contributions:
\begin{itemize}
\item We show that for an autonomous DS the backward-reachable-tube value function reduces to
a deterministic minimum of the safety function along the rollout, eliminating the
optimal-control term that renders reachability intractable for general systems.

\item We derive, from stability Lyapunov certificate with a known convergence rate, a closed-form
finite-horizon truncation of the value function, and prove that forward invariance over this
finite horizon implies invariance for all time.

\item We prove that the resulting safe set is the maximal forward-invariant subset of the
obstacle-free region under the nominal flow, and is therefore the least conservative such
certificate.

\item We show that this certificate removes a failure mode shared by modulation and
geometric-CBF–based avoidance: the head-on stagnation point that becomes a saddle equilibrium
under reference-based modulation and a spurious attractor under a geometric CBF.
\item We instantiate the certificate across five DS formulations and validate it on a 7-DoF
Franka manipulator.
\end{itemize}

\section{Background}
\subsection{Dynamical Systems (DS)}
A dynamical system (DS) \cite{DSbook} is a first-order, autonomous system used as a motion plan for a robotic system
\begin{align}
    \dot{x} = f(x) 
\end{align} where $x\in\mathbb{R}^n$ is the state of the system and $f(\cdot):\mathbb{R}^n\rightarrow\mathbb{R}^n$ defines a vector field over the entire state space. Autonomous here means
$f$ depends only on the state, with no explicit time or external input; integrating the field from any initial state yields a trajectory that imitates the demonstrations. 

\subsection{Lyapunov Conditions for Stability}
Stability of a DS is established through the existence of a
Lyapunov function $V_L(x):\mathbb{R}^n\rightarrow\mathbb{R}$ \cite{lasa, KhansariZadeh2014LearningCL}. Without loss of generality, we shift coordinates
so that the equilibrium is at the origin, $x^* = 0$. A Lyapunov function
satisfies:
\begin{enumerate}
    \item $V_L(0) = 0$
    \item $V_L(x) > 0$ for all $x \neq 0$
    \item $\dot{V}_L(x) \le 0$ for all $x \neq 0$
\end{enumerate}
where $\dot{V}_L(x) = \nabla V_L(x) \cdot f(x)$ is the rate of change of $V_L$
along trajectories. Such a function certifies \emph{Lyapunov stability}:
a trajectory starting within a neighborhood of the origin remains within a
bounded neighborhood of it for all time. Strengthening condition (3) to a
strict inequality, $\dot{V}_L(x) < 0$ for all $x \neq 0$, certifies
\emph{asymptotic stability}, under which the trajectory additionally
converges to the equilibrium, $x(t) \to 0$ as $t \to \infty$, though at no
guaranteed rate. The strongest of the three is \emph{exponential stability}, certified by
\begin{align}
    \dot{V}_L(x) \le -\alpha V_L(x)
    \quad \text{and} \quad
    c_1 \|x\|^2 \le V_L(x) \le c_2 \|x\|^2,
    \label{eqn_stability}
\end{align}
where $\alpha > 0$ is the convergence rate and $c_1, c_2 > 0$ bound $V_L$
between two paraboloids. Exponential stability implies convergence of the state to the equilibrium at an exponential rate
\begin{align}
    \|x(t)\| \le \sqrt{c_2/c_1}\, \|x_0\|\, e^{-\alpha t/2}.
\end{align}
This exponential envelope implies that for any target tolerance $\eta$, the
trajectory satisfies $\|x(t)\| \le \eta$ within a finite time. 
\subsection{Forward Invariance and Control Barrier Functions}

For motion planning, safety is the notion where the trajectory never
enters an unsafe zone. We formalize this through forward invariance. A set
$S$ is \emph{forward invariant} under nominal dynamics $\dot{x} = f(x)$ if every trajectory
starting in $S$ remains in $S$ for all time,
\begin{align}
    x_0 \in S \implies x(t) \in S \quad \text{for all } t \ge 0.
\end{align}
If the failure region lies outside $S$, forward invariance of $S$ is
the statement that the system stays safe.

Let $S$ be the $0$-superlevel set of a continuously differentiable function
$h : \mathbb{R}^{n} \to \mathbb{R}$, $S = \{x : h(x) \ge 0\}$, with
$\nabla h(x) \neq 0$ on $\partial S$. A \emph{barrier function}
\cite{Ames_2017, cbftut} certifies invariance of $S$ through a pointwise condition on the rate of
change of $h$ along the flow,
\begin{align}
    \dot{h}(x) = \nabla h(x) \cdot f(x) \ge -\alpha\, h(x),
    \qquad \alpha > 0.
    \label{eq:cbf}
\end{align}
Enforcing \eqref{eq:cbf} from $h(x_0) \ge 0$ ensures the trajectory never
leaves $S$. i.e., condition \eqref{eq:cbf} certifies forward invariance of $S$.

\subsection{Backward Reachability and the Value Function}
Herein, we relate our value function\footnote{\nf{Note that this value function $V(x)$ is not the same as the Lyapunov function $V_L(x)$ used to ensure stability of the nominal DS.}} $V(x):\mathbb{R}^n\rightarrow\mathbb{R}$ to the backward-reachability concept. Consider a controlled-system $\dot{x} = f(x, u)$ and a failure set described by a safety function $h: \mathbb{R}^n\to \mathbb{R}$, whose zero-sublevel set is the failure region, $\mathcal{F} = \{x : h(x) < 0\}$. The \emph{backward-reachability tube} (BRT) \cite{bansal2017hamiltonjacobireachabilitybriefoverview, brt1} for this failure region is the set from where the system would inevitably go into $\mathcal{F}$ despite the best admissible control efforts. 
\begin{align}
    V(x, t) = \sup_{u(\cdot)} \; \min_{\tau \in [t, T]} h\big(x(\tau)\big),
    \label{eq:brt-value}
\end{align}
Here, $x(\tau)$ is the trajectory under control $u$ from $x$. The reachability value function $V(x,t) \ge 0$ so defined certifies that a control law exists that keeps the entire trajectory in the safe region, and the safe set is the superlevel set $\{x : V(x,t) \ge 0\}$. 

This value function is classically optained through Hamilton Jacobi Reachability Analysis where this function is obtained as a viscosity solution to the Hamilton Jacobi PDE \cite{choi2021robustcontrolbarriervaluefunctions}. This route suffers from curse of dimensionality where the solution becomes intractable as the problem dimensionality increases, limiting the application of such a reachability function in robotics problems. A central property is that its superlevel set is the \emph{maximal} control-invariant subset of the safe
region; i.e., every control-invariant set is contained in it, so $V$ is the least
conservative safety certificate available to the controlled system
\cite{choi2021robustcontrolbarriervaluefunctions}.

\section{Methodology}
We treat the controlled backward-reachability value function of \eqref{eq:brt-value} to dynamical systems (DS) motion policies. We make the following assumptions on the DS under consideration:

\begin{enumerate}
    \item \textbf{Attractor.} The DS $\dot{x} = f(x)$ is autonomous and has
    a single equilibrium $x^*$ toward which trajectories converge, \nf{asymptotically or exponentially.}
    \item \textbf{Observable finite-time convergence.} There exists a finite, computable time $T_{fin}(x_0)$ after which the trajectory enters $\eta$-ball about $x^*$. $T_{fin}(x_0)$ can generally be obtained from exponential stability guarantee imbibed in the DS or through construction of the DS. 
    \item \textbf{Attractor outside failure state.} The failure region is away from the equilibrium; i.e., $h(x) > 0 ~ \forall x$ within distance $\eta$ of
    $x^*$.
\end{enumerate}

These assumptions generally hold for a broad class of stable DS learned from
demonstration in literature.

\subsection{Autonomy Collapse of the Value Function}

The value function from \eqref{eq:brt-value} is an optimal control problem with a supremum over the admissible control input. The safety of the state $x$ depends on whether an admissible control input is present that can keep the trajectory outside the failure state. This is the source of the HJ PDE and one of the reasons of computation intractability.  

For an autonomous DS there is no control input, from any state $x$ the
trajectory $\phi_\tau(x)$ is obtained by integrating $\dot{x} = f(x)$. Therefore, the optimization problem reduces to the minimization of $h$ over time which is, depending on the setting, either a closed form equation or easily determined:
\begin{align}
    V(x, t) = \min_{\tau \in [t, T]} h\big(\phi_\tau(x)\big).
    \label{eq:collapsed-value}
\end{align}

Setting $t = 0$ and extending the horizon to infinity yields the stationary
value function
\begin{align}
    V(x) = \min_{\tau \in [0, \infty)} h\big(\phi_\tau(x)\big),
    \label{eq:stationary-value}
\end{align}
which yields the worst safety value over the entire future of
the nominal trajectory from $x$.

Two consequences follow directly from \eqref{eq:stationary-value}. First, $V(x) \ge 0$ if and only if $h(\phi_\tau(x)) \ge 0$ for all
$\tau \ge 0$, that is, the safe set $\{V \ge 0\}$ is exactly the set of
states whose nominal trajectory never enters the failure region. Safety of
the nominal flow over $\{V \ge 0\}$ is thus immediate from the definition,
inducing forward invariance property by definition. Second, because $V$ is built from
the rollout of the deterministic flow, it can be evaluated by integration
alone which can be done in closed form where $\phi_\tau$ is available analytically, and by
a learned approximation otherwise without ever
solving a partial differential equation.

\subsection{Finite-Horizon Truncation}
The stationary value function $V(x)$ requires a rollout over an infinite time horizon, which is not tractable. Assumption~2 guarantees a finite, computable time $T_{\mathrm{fin}}(x_0)$ after which the trajectory has entered and thereafter remains within an $\eta$-ball about $x^*$. This is the only quantity the truncation needs, and it lets us replace the infinite rollout with a finite one that is provably equivalent for the purpose of certifying safety.

\subsubsection{Truncation time}
A finite $T_{\mathrm{fin}}$ is available in one of two ways, depending on how the
DS is specified.

\paragraph{Through exponential stability.}
When the DS is exponentially stable, a trajectory starting at $x_0$ satisfies
the envelope $\|\phi_\tau(x_0) - x^*\| \le \sqrt{c_2/c_1}\,\|x_0 - x^*\|\,
e^{-\alpha_L \tau / 2}$, where $\alpha_L$ is the convergence rate and $c_1,c_2$
bound the Lyapunov function. Requiring $\|\phi_\tau(x_0) - x^*\| \le \eta$ and
solving for $\tau$ gives the horizon in closed form,
\begin{align}
    T_{\mathrm{fin}}(x_0)
    = \frac{2}{\alpha_L}\,
      \ln\!\left( \frac{\sqrt{c_2/c_1}\,\|x_0 - x^*\|}{\eta} \right).
    \label{eq:tfin}
\end{align}

\paragraph{Through construction.}
Alternatively, some policies are built so that a finite horizon holds by
construction, without an explicit rate estimate. This is the case for the
diffeomorphic latent-space DS of Section~\ref{latent}, whose convergence
is inherited from a stable latent template;

In either case $T_{\mathrm{fin}}(x_0)$ is finite and computable, which is what
the following construction requires.

\subsubsection{Equivalence of the truncated value function}
By Assumption~3 the failure set is separated from the equilibrium, i.e.,
$h(x) > 0 ~\forall x$ within distance $\eta$ of $x^*$. Once the trajectory
enters the $\eta$-ball it is therefore safe for all remaining time, so the
worst-case safety value occurs before $T_{\mathrm{fin}}$. This motivates the
truncated value function
\begin{align}
    V(x) = \min_{\tau \in [0,\, T_{\mathrm{fin}}(x)]}
           h\big(\phi_\tau(x)\big),
    \label{eq:truncated-value}
\end{align}
which we now show certifies safety over the infinite horizon.

\begin{lemma}[Finite-horizon truncation]
\label{lem:truncation}
Let the DS satisfy Assumptions 1--3, and let $V$ be the truncated value
function \eqref{eq:truncated-value}. Then for every state $x$,
\begin{align}
    V(x) \ge 0
    \quad \Longleftrightarrow \quad
    h\big(\phi_\tau(x)\big) \ge 0 \;\; \text{for all } \tau \ge 0 .
\end{align}
That is, certifying $V(x) \ge 0$ over the finite horizon
$[0, T_{\mathrm{fin}}(x)]$ certifies safety of the entire infinite-horizon
trajectory.
\end{lemma}

\begin{proof}
Forward direction ($\Leftarrow$). If the trajectory is safe for all
$\tau \ge 0$, then the minimum over the subinterval $[0, T_{\mathrm{fin}}]$ is
non-negative, so $V(x) \ge 0$.

Reverse direction ($\Rightarrow$). Partition the trajectory at
$T_{\mathrm{fin}}$. On $[0, T_{\mathrm{fin}}]$, $V(x) \ge 0$ gives
$h(\phi_\tau(x)) \ge 0$ by \eqref{eq:truncated-value}. For
$\tau > T_{\mathrm{fin}}$, the definition of the truncation time places the
trajectory inside the $\eta$-ball about $x^*$, where $h > 0$ by Assumption~3.
Hence $h(\phi_\tau(x)) \ge 0$ on both intervals, i.e.\ for all $\tau \ge 0$.
\end{proof}

Lemma~\ref{lem:truncation} reduces an infinite-horizon safety certificate to a
finite rollout whose length is $T_{\mathrm{fin}}$ which is given in closed form by
\eqref{eq:tfin} under exponential stability, and by construction for the DS
families that guarantee it otherwise. This is what makes $V$ computable, it can be derived analytically when $\phi_\tau$ is known, or learnable otherwise, and it is the property the remaining constructions rely on.

\subsection{Maximality of the Safe Set}

We now characterize the safe set $\{x : V(x) \ge 0\}$ induced by the value
function. For the controlled system, the superlevel set of the
reachability value function is the maximal control-invariant subset of the
safe region \cite{choi2021robustcontrolbarriervaluefunctions}. We show that
the autonomous DS inherits an analogous property,
which makes the resulting certificate the least conservative \nf{compared to the state-of-the art approaches (modulation and CBFs).} 

Recall the obstacle-free region $\mathcal{S}_h = \{x : h(x) \ge 0\}$, the
raw set of states not in collision. Let
$\mathcal{S}_V = \{x : V(x) \ge 0\}$ be the safe set induced by the value
function \eqref{eq:truncated-value}.

\begin{proposition}[Maximal forward-invariant set]
\label{prop:maximal}
Under Assumptions 1--3, $\mathcal{S}_V$ is the maximal forward-invariant
subset of $\mathcal{S}_h$ under the nominal DS flow $\dot{x} = f(x)$:
\begin{enumerate}
    \item[(i)] $\mathcal{S}_V \subseteq \mathcal{S}_h$, and $\mathcal{S}_V$
    is forward invariant under $f$;
    \item[(ii)] every forward-invariant subset of $\mathcal{S}_h$ is
    contained in $\mathcal{S}_V$.
\end{enumerate}
\end{proposition}

\begin{proof}
\emph{(i)} If $V(x) \ge 0$ then, by Lemma~\ref{lem:truncation},
$h(\phi_\tau(x)) \ge 0$ for all $\tau \ge 0$; in particular at $\tau = 0$,
$h(x) \ge 0$, so $x \in \mathcal{S}_h$ and $\mathcal{S}_V \subseteq
\mathcal{S}_h$. For invariance, fix $x \in \mathcal{S}_V$ and any $s \ge 0$.
The trajectory from $\phi_s(x)$ is the tail of the trajectory from $x$, so
$h(\phi_\tau(\phi_s(x))) = h(\phi_{s+\tau}(x)) \ge 0$ for all $\tau \ge 0$;
hence $V(\phi_s(x)) \ge 0$ and $\phi_s(x) \in \mathcal{S}_V$. Thus
$\mathcal{S}_V$ is forward invariant.

\emph{(ii)} Let $\mathcal{I} \subseteq \mathcal{S}_h$ be forward invariant
under $f$, and let $x \in \mathcal{I}$. By invariance $\phi_\tau(x) \in
\mathcal{I} \subseteq \mathcal{S}_h$ for all $\tau \ge 0$, so
$h(\phi_\tau(x)) \ge 0$ for all $\tau \ge 0$, giving $V(x) \ge 0$ and
$x \in \mathcal{S}_V$. Hence $\mathcal{I} \subseteq \mathcal{S}_V$.
\end{proof}

By (i), $\mathcal{S}_V$ is a
strict subset of the obstacle-free region, as the gap
$\mathcal{S}_h \setminus \mathcal{S}_V$ consists of states that are
instantaneously collision-free but lie on nominal trajectories that eventually enter the failure set. By (ii), no forward-invariant subset of
$\mathcal{S}_h$ extends beyond $\mathcal{S}_V$, so $\mathcal{S}_V$ is the
largest set that can be rendered safe under the nominal flow.

\begin{corollary}[Minimal conservativeness]
\label{cor:conservative}
Any safety certificate whose safe set is a forward-invariant subset of
$\mathcal{S}_h$ is contained in $\mathcal{S}_V$. Consequently, a filter that
enforces invariance of $\mathcal{S}_V$ intervenes on no state that a
correct certificate would admit as safe, thus, it is the least conservative certificate.
\end{corollary}

Physically, a barrier on 
$h$ reacts to instantaneous proximity to the obstacle surface and may admit states that eventually lead to collision while influencing states that were never at risk. 

\subsection{The Safety Filter}
When the value function indicates an unsafe trajectory --- either $V(x) < 0$, or a perturbation pushes the state
toward $\partial\mathcal{S}_V$ --- we modify the DS through a virtual input
$u$, $\dot{x} = f(x) + u$, to keep the trajectory safe.

In this scenario, we use $V$ as a \emph{barrier function} for $f + u$. By
Proposition~\ref{prop:maximal}, $\{V \ge 0\}$ is the maximal invariant set
of the nominal flow, hence the least conservative superlevel set on which to
enforce a barrier condition. We obtain $u$ as the minimum-norm correction
maintaining $\dot{V} \ge -\alpha_V V$:
\begin{align}
    \min_{u} \;\; & \tfrac{1}{2}\|u\|^2 \nonumber\\
    \text{s.t.} \;\; & \nabla V(x)^\top (f(x) + u) + \alpha_V V(x) \ge 0,
    \label{eq:qp-safety}
\end{align}
with barrier gain $\alpha_V > 0$. 

This optimization problem also induces a convenient switching behavior between the filter being inactive and active. Along a nominal trajectory with $V(x) > 0$, the value function is non-decreasing. Thus, $V$ is constant until the trajectory reaches its closest approach point to the obstacle and increases once that point is passed and the trajectory recedes. Hence $\dot{V} \ge 0 > -\alpha_V V$, the constraint \eqref{eq:qp-safety}
holds for nominal dynamics, and $u = 0$. The filter is therefore off over the entire safe set under nominal DS flow, including trajectories that pass close to the obstacle, consistent with Corollary~\ref{cor:conservative}.

A nonzero correction arises only in two cases: (i) when a perturbation drives the state toward $\partial\mathcal{S}_V$, or (ii) when the state lies in the unsafe set $V < 0$. In the first, $u$ acts to keep $\dot{V} \ge -\alpha_V V$ and preserve invariance --- a guarantee of \emph{maintenance}. In the
second, the same condition drives $V$ upward toward the safe set in a recovery behavior.

Where stability is not guaranteed by the DS construction
\eqref{eqn_stability}, the analogous Lyapunov condition on $V_L$ is added as
a second constraint, relaxed by a slack variable $\delta$ so that safety takes
precedence:
\begin{align}
    \min_{u, \delta } \;\; & \tfrac{1}{2}\|u\|^2 + w\cdot \delta ^2 \nonumber\\
    \text{s.t.} \;\; & \nabla V(x)^\top (f(x) + u) + \alpha_V V(x) \ge 0 \nonumber,\\
     & \nabla V_L(x)^\top (f(x) + u) + \alpha_{V_L} V_L(x)\le \delta \nonumber,\\
    & \delta \ge0
    \label{eq:qp-safety1}
\end{align}

with barrier gain $\alpha_V$ retaining its function for the reachability value function $V$ and $\alpha_{V_L}$ acting as the barrier gain for the Lyapunov function $V_L$.  

\section{Obstacle-Avoidance for Concave Obstacles}
Modulation-based approaches \cite{DSbook,Huber2019-ur,onmanifold} are closed-form solutions tailored to reshape a nominal DS to avoid obstacles or constraints. However, they are known to suffer from saddle points (or undesirable equilibria) when avoiding the obstacles with both convex and concave geometries \cite{onmanifold}. This drawback is also found in geometric CBF collision avoidance strategies \cite{mcbf}; and is overall a common issue in geometric-based reactive obstacle avoidance techniques. In this section, we show that our proposed reachability-based safety filter \eqref{eq:qp-safety} driven by our value function $V(x)$ alleviates this issue altogether. 

\subsection{Modulation and CBF as DS Safety Filters}
Next, we summarize the two common safety filters used for DS motion planning: (i) modulation and (ii) geometric CBFs. 

\subsubsection{Modulation} For reference direction-based approach with $\Gamma(x):\mathbb{R}^n\rightarrow\mathbb{R}$ as the distance function, the modulated DS takes the form:
\begin{align}
    \dot x = M_r(x)\,f(x), \qquad M_r = E_r\,D_r\,E_r^{-1},
\end{align}
with the non-orthonormal basis $E_r(x)=[\,r(x),\,e(x)\,]$, reference direction
$r(x)=\tfrac{x-x^r}{\|x-x^r\|}$ ($x^r$ inside the obstacle), tangents $e(x)$,
and eigenvalues
\begin{align}
    \lambda_n = 1-\tfrac{1}{\Gamma(x)}, \qquad
    \lambda_t = 1+\tfrac{1}{\Gamma(x)},
\end{align}
which compress the normal and stretch the tangential component so the boundary
is never crossed. Expanded,
\begin{align}
    \dot x = n\lambda_n\langle n,f\rangle + e\lambda_t\langle e,f\rangle
             -\tfrac{2}{\Gamma}\,e\,
             \tfrac{\langle e,r\rangle}{\langle n,r\rangle}\langle n,f\rangle .
\end{align}

For head-on collisions, $f \parallel n$ and hence $\dot{x} = 0$ at the collision point on the boundary which acts as a saddle point in this modulated DS. While recent approaches have addressed this issue through contaction theory and differential geometry insights, their proofs are only valid for linear nominal DS motion policies \cite{onmanifold,Huber2019-ur}. As studied in \cite{koptev}, for arbitrary DS motion policies with arbitrary high-dimensional obstacle geometries (like joint space obstacle representations) one requires a planner or heuristic to avoid such saddle points. 

\subsubsection{Geometric CBF}  Geometric CBF approaches, on the other hand, filter the nominal DS through a barrier
$h(x)=\Gamma(x)-1$ ($h\ge 0$ safe), solving the min-intervention QP
\begin{align}
    u^\star = \arg\min_{u}\ \|u-f(x)\|^2
    \quad\text{s.t.}\quad \nabla h^\top u \ge -\alpha\,h,
\end{align}
whose closed-form safe velocity, when the constraint is active, is the
projection
\begin{align}
    \dot x = u^\star
    = f - \frac{\nabla h^\top f + \alpha h}{\|\nabla h\|^2}\,\nabla h,
\end{align}
removing only the inward normal component of $f$. On the boundary ($h=0$) a
head-on collision has $f\parallel\nabla h$, so the projection cancels $f$
entirely and $\dot x = 0$. Unlike modulation, the barrier never injects a
tangential escape velocity, so this stagnation point can often
act as a stable \emph{attractor} too as seen in Figure \ref{fig:saddle} and studied in depth in \cite{mcbf}.

\begin{figure*}[t]
  \centering
  \includegraphics[width=\textwidth]{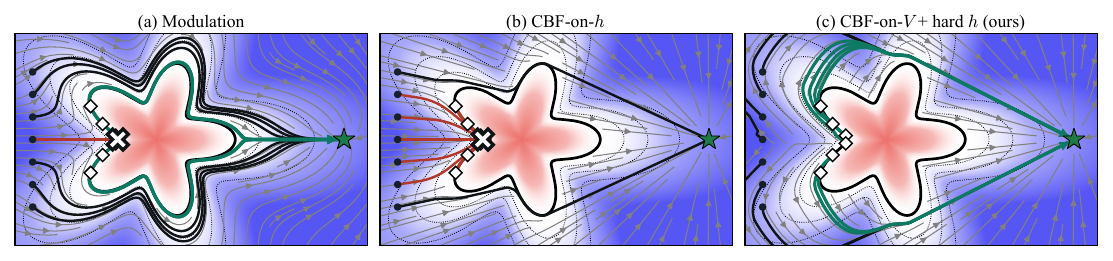}
  \caption{Head-on behavior on a star obstacle. Background shading is the safety
    margin $h=\Gamma-1$ (red: obstacle interior, blue: free space, opacity:
    distance to the boundary). Dots are far initial conditions, diamonds are on-obstacle
    initial conditions on the head-on arc, $\times$ marks the spurious
    equilibrium, and the star is the goal; black/teal trajectories reach the
    goal, red trajectories are trapped.
    \textbf{(a)} Reference modulation: the head-on line and one surface seed fall
    into a \emph{saddle}. \textbf{(b)} CBF-on-$h$: a spurious \emph{attractor}
    whose basin traps the far starts and the whole notch arc of surface seeds.
    \textbf{(c)} Stacked CBF-on-$V$ + hard-$h$ (ours): every start, on-obstacle
    or far, is steered around to the goal while keeping $h\ge 0$.}
  \label{fig:saddle}
\end{figure*}

\subsection{Saddle Point Removal w. Stacked CBF-on-$V$ + CBF-on-$h$}

Both failures highlighted in the previously come from myopic safety filtering; i.e., acting on where the trajectory is \emph{now}. To
remedy this, we stack two constraints in the filter,
\begin{align}
    u^\star & = \arg\min_{u}\ \|u-f(x)\|^2 \nonumber\\ \text{s.t.}&\quad
    \nabla h^\top u \ge -\alpha_h\,h\\ \nonumber
    & \quad \nabla V^\top u \ge -\alpha_V\,V
\end{align}
the first keeping the state out of the obstacle (as in modulation or CBF approaches), the second keeping its future
rollout safe. 
\paragraph{Tangential escape on the obstacle}
At a stall point on the boundary ($h=0$, $\nabla h^\top f<0$) the nominal
rollout enters the obstacle, so $V<0$. At $h=0$ the first constraint is
$\nabla h^\top u\ge 0$, which blocks inward motion, and the second constraint gives
$\nabla V^\top u\ge -\alpha_V V>0$. So $u^\star=0$ is infeasible. The only motion
that raises $V$ without going inward is along the boundary, so the state moves
off tangentially instead of stalling. 

\paragraph{Discontinuity.}
The gradient of $V$ depends on the minimizing time
$\tau^\star(x)=\arg\min_{\tau\in[0,T_{\mathrm{fin}}]} h(\phi_\tau(x))$,
\begin{align}
    \nabla V(x) = \big(\partial_x \phi_{\tau^\star}(x)\big)^\top
                  \nabla h\big(\phi_{\tau^\star}(x)\big).
\end{align}
Behind the obstacle, there might be the existence of a set $\Sigma$ where the rollout clears the
obstacle to either side, so $\tau^\star$ is not unique,
$\tau^\star\in\{\tau_L,\tau_R\}$, and $\nabla V_L\neq\nabla V_R$. We break the
tie by fixing a side, $\nabla V\coloneqq\nabla V_L$. The safe velocity is then
discontinuous across $\Sigma$,
\begin{align}
    \lim_{x\to\Sigma^{+}} u^\star \neq \lim_{x\to\Sigma^{-}} u^\star .
\end{align}
This does not affect convergence and trajectories do not stay on it. A comparison is drawn in Figure \ref{fig:saddle} which shows the behavior for a single star-shaped obstacle across these methods. 

\section{\nf{Applications to Different DS Constructions}}
We implement the certificate across five DS constructions. We first verify the method on analytical DS, where $V$ is
available in closed form and the result can be checked directly
(Section V-A). We then demonstrate the \emph{learned} construction on a
Neural ODE, the setting that scales to systems with no closed-form rollout
(Section V-B). We show \emph{generality} on diffeomorphic latent-space,
LPV-DS (Sections V-C and V-D) and finally \emph{bridge to hardware}
through an SE(3) formulation deployed on a manipulator (Section V-E).

\subsection{Analytical DS - Spiral and Linear DS}
\subsubsection{Linear DS}
Consider a linear DS $\dot{x} = A(x-x^*)$ with $A$ symmetric negative
definite, whose trajectories converge to the equilibrium $x^*$. The flow has a closed form equation, 
\begin{align}
    \phi_\tau(x_0) = x^* + e^{A\tau}(x_0 - x^*),
\end{align}
and we take the quadratic Lyapunov function $V_L(x) = (x - x^*)^\top P (x -
x^*)$. We take the safety function as the signed distance to a circular obstacle
of radius $r$ centered at $x_{\mathrm{obs}}$,
$h(x) = \|x - x_{\mathrm{obs}}\| - r$. The resultant value function
\eqref{eq:truncated-value} is
\begin{align}
    V(x_0) = \min_{\tau \in [0,\, T_{\mathrm{fin}}]}
             \Big( \|\phi_\tau(x_0) - x_{\mathrm{obs}}\| - r \Big),
\end{align}
where the horizon $T_{\mathrm{fin}}$ is given by \eqref{eq:tfin} with
$c_1, c_2$ the extremal eigenvalues of $P$ and $\alpha_L$ the contraction
rate of $A$. Evaluating the gradient at the minimizing time $\tau^*$,
\begin{align}
    \nabla_{x_0} V
    = \big(e^{A\tau^*}\big)^\top
      \frac{\phi_{\tau^*}(x_0) - x_{\mathrm{obs}}}
           {\|\phi_{\tau^*}(x_0) - x_{\mathrm{obs}}\|}.
\end{align}

\subsubsection{Spiral DS}
Taking
$A = \begin{bsmallmatrix} -\sigma & \omega \\ -\omega & -\sigma
\end{bsmallmatrix}$
yields a spiral DS converging to $x^*$, with closed-form trajectory
\begin{align}
    \phi_\tau(x_0) = x^* + e^{-\sigma\tau} R(-\omega\tau)(x_0 - x^*),
\end{align}
where $R(\cdot)$ is the planar rotation matrix. Using the Lyapunov function
$V_L(x) = (x - x^*)^\top(x - x^*)$ and the same signed-distance safety
function, the value function is
\begin{align}
    V(x_0) = \min_{\tau \in [0,\, T_{\mathrm{fin}}]}
             \Big( \|\phi_\tau(x_0) - x_{\mathrm{obs}}\| - r \Big),
\end{align}
with gradient at $\tau^*$
\begin{align}
    \nabla_{x_0} V
    = e^{-\sigma\tau^*} R(-\omega\tau^*)^\top
      \frac{\phi_{\tau^*}(x_0) - x_{\mathrm{obs}}}
           {\|\phi_{\tau^*}(x_0) - x_{\mathrm{obs}}\|}.
\end{align}

In both cases $V$ and $\nabla_{x_0} V$ are closed form, and we obtain the
safe velocity directly from the filter \eqref{eq:qp-safety}.
Figure ~\ref{fig:main} shows the resulting trajectories: the nominal flow
is corrected initially till it reaches a state where the nominal DS avoids the obstacle. Thereafter, there is no correction illustrating the minimum-intervention switching behavior of the filter. 

\begin{figure*}[t]
  \centering
  \subfloat[Analytical (spiral)]{%
    \includegraphics[width=0.19\textwidth]{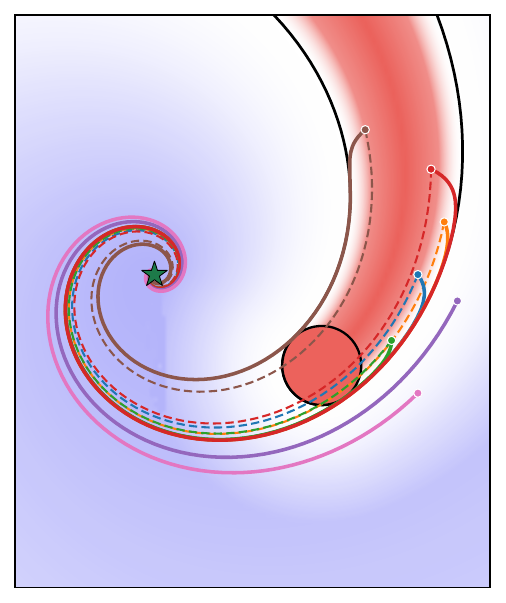}}\hfil
  \subfloat[Neural ODE]{%
    \includegraphics[width=0.19\textwidth]{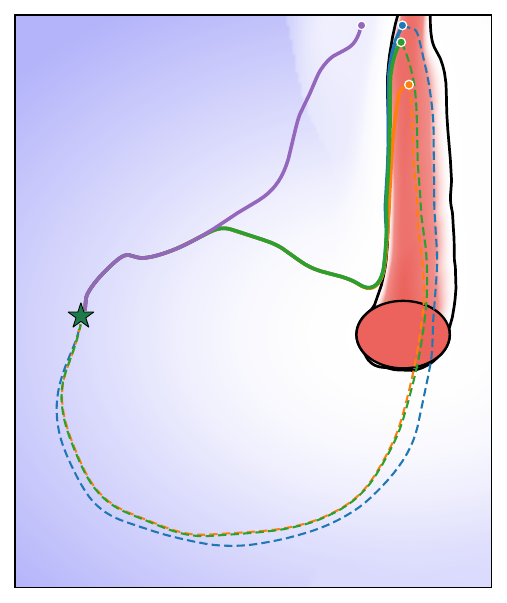}}\hfil
  \subfloat[LPV-DS]{%
    \includegraphics[width=0.19\textwidth]{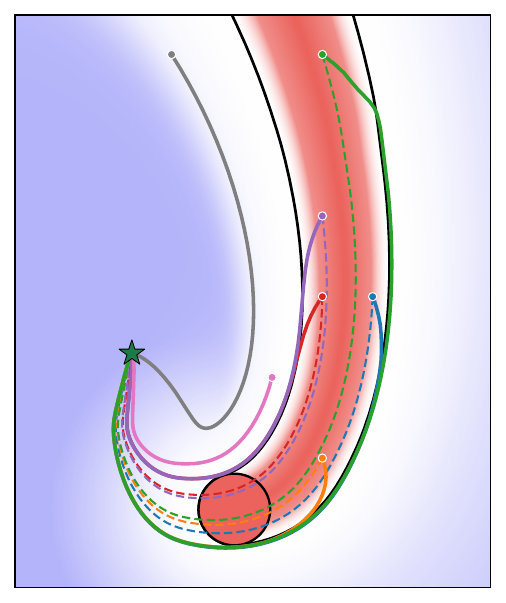}}\hfil
  \subfloat[Latent space]{%
    \includegraphics[width=0.19\textwidth]{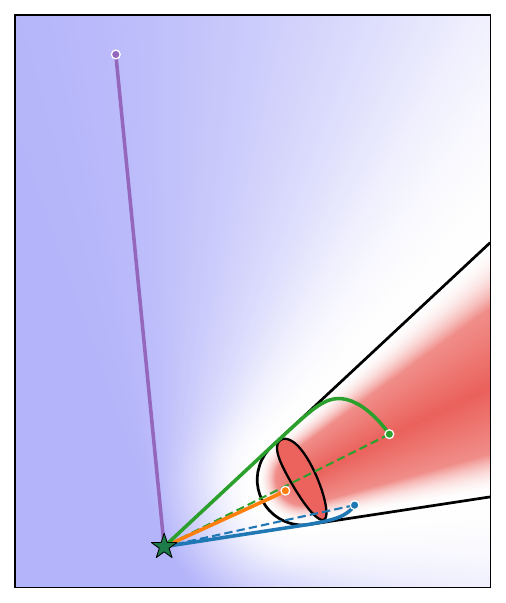}}\hfil
  \subfloat[Task space]{%
    \includegraphics[width=0.19\textwidth, height=0.2265\textwidth]{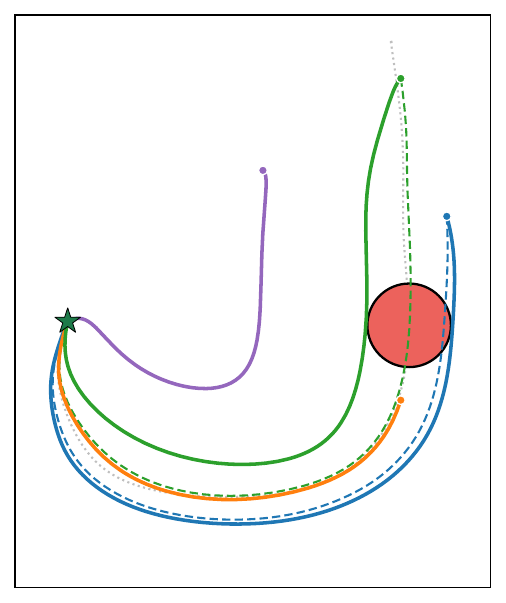}}
  \caption{Main results across the four DS constructions. In each panel the
    nominal (dashed) and filtered (solid) rollouts are overlaid on the
    reachability value landscape $V$, with the obstacle marked; only the
    computation of $V$ changes between panels, the filter
    (Eq.~\eqref{eq:qp-safety}) is identical. (a)~Spiral DS with closed-form $V$:
    the nominal arc enters the obstacle, the filtered rollout rides the
    $\{V=0\}$ boundary and converges. (b)~Neural-ODE DS with a learned value
    network on a LASA shape (JShape): four initial conditions shown, three with
    $V(x_0)<0$ and one with $V(x_0)\ge0$, each nominal/filtered pair in the same
    colour. (c)~LPV-DS with the certified exponential rate of
    $\alpha_L = \min_k \lambda_{\min}(M_k)\,/\,\lambda_{\max}(P)$ as noted in section \ref{sec:lpvds}. (d)-(e)~Diffeomorphic latent DS: the
    certificate is enforced in latent space (d) and the same four rollouts are
    shown pulled back through $\psi^{-1}$ into task space (e).}
  \label{fig:main}
\end{figure*}

\subsection{Neural ODE}
\label{sec:node}

We now turn to the case with a DS with no closed-form rollout, where we must learn the value function. We build on a Neural ODE learned jointly with
an ICNN (Input Convex Neural Network) Lyapunov function $V_L$ \cite{ pmlr-v202-xiao23d, manek2020learningstabledeepdynamics, 10610791}, which
certifies exponential stability, illustrating the capability of this method to extend to a learning framework.

The DS is rendered affine in its last-layer parameters
$\theta_{K-1,K}$. The original construction
enforces only the CLF (stability) condition; we add the value-function
condition, giving the projection
\begin{align}
    \min_{\theta_{K-1,K}} \;\;
        & \|\theta_{K-1,K} - \theta^{\dagger}_{K-1,K}\|^2 \nonumber\\
    \text{s.t.}\;\; & \nabla V_L^\top \dot{x}(\theta_{K-1,K}) + \alpha_L V_L \le 0,
    \label{eq:node-clf}\\
        & \nabla V^\top \dot{x}(\theta_{K-1,K}) + \alpha_V V \ge 0,
    \label{eq:node-cbf}
\end{align}
where $\theta^{\dagger}$ is the unconstrained last layer and $\dot{x}$ is
linear in $\theta_{K-1,K}$. Since there is no closed form $\phi_\tau$, it is not possible to find a closed form equation for $V$. We learn it with an
MLP $V_\theta$ regressed on rollout targets $\{x_0,\,
\min_{\tau\in[0,T_{\mathrm{fin}}]} h(\phi_\tau(x_0))\}$. Exponential stability makes this tractable.

Figure~\ref{fig:main} shows the learned $V$ and filtered
trajectories on Jshape of the LASA dataset \cite{lasa}.

\subsection{Latent-Space DS} \label{latent}

To show that the construction is not tied to any particular parameterization, we
use it on a diffeomorphic latent-space DS \cite{euclideanizing, stochasticDS}. Such methods learn a bijection
$\psi : \mathbb{R}^n \to \mathbb{R}^n$ mapping the observation space to a
latent space in which the dynamics are a simple linear stable flow. We apply the certificate in the latent space and map the result back via $\psi$.

In latent space, the DS is a unit-speed radial decay to the origin,
\begin{align}
    \dot{y} = -\frac{y}{\|y\|},
    \qquad
    \phi_\tau(y_0) = y_0 \max\!\big(1 - \tau/\|y_0\|,\, 0\big),
\end{align}
which reaches $y^\ast = 0$ exactly at $\tau = \|y_0\|$. The horizon is thus
exact, $T_{\mathrm{fin}} = \|y_0\|$, with no truncation needed.

The obstacle is mapped into latent space by sampling its observation-space
boundary $\partial\mathcal{O}$, pushing the samples through $\psi$, and
fitting an enclosing circle $C(y_c, r_y)$. Because $\psi$ is a
diffeomorphism, a trajectory clearing the latent obstacle clears the real
one. As the latent trajectory is a
straight line toward the origin, the value function is the closed-form
signed distance from that line to the latent obstacle,
\begin{align}
    V(y_0) = \|y_0 - \tau^\ast \hat{e} - y_c\| - r_y,
    \qquad \hat{e} = -\frac{y_0}{\|y_0\|},
\end{align}
where $\tau^\ast \hat{e}$ is the point of closest approach along the
trajectory. The filter \eqref{eq:qp-safety} is applied in latent space and
the corrected velocity mapped back through $\psi$. Figure ~\ref{fig:main}
shows the resulting obstacle avoidance in the observation space.

\subsection{LPV-DS}
\label{sec:lpvds}
We apply the certificate to an LPV-DS \cite{lpvds,DSbook} , a learned policy written as a state-dependent mixture of linear systems. Taking the equilibrium at the origin
($x^* = 0$), the dynamics read
\begin{align}
    \dot{x} = \sum_{k=1}^{K}\gamma_k(x)\, A_k x,
    \qquad \gamma_k \ge 0,\ \ \textstyle\sum_k \gamma_k = 1.
\end{align}
Stability is certified at learning time by a common quadratic Lyapunov function
$V_L(x) = x^\top P x$, $P \succ 0$, with the strict decrease constraints
$A_k^\top P + P A_k \preceq -Q_k \prec 0$. Writing
$M_k \coloneqq -(A_k^\top P + P A_k) \succ 0$ and using $\sum_k \gamma_k = 1$,
\begin{align}
    \dot{V}_L
    &= -\sum_k \gamma_k\, x^\top M_k\, x \nonumber\\
    &\le -\Big(\min_k \lambda_{\min}(M_k)\Big)\|x\|^2
       \;\le\; -\alpha_L V_L,
\end{align}
with $\alpha_L = \min_k \lambda_{\min}(M_k)\,/\,\lambda_{\max}(P)$. The LPV-DS
is thus exponentially stable at a rate read directly off the learned model, and
the horizon $T_{\mathrm{fin}}$ follows from \eqref{eq:tfin} with
$c_1 = \lambda_{\min}(P)$ and $c_2 = \lambda_{\max}(P)$

Unlike the analytical DS, the mixture admits no closed-form flow, so we obtain
$V$ by numerically integrating $\phi_\tau$ over $[0, T_{\mathrm{fin}}]$ and
taking the running minimum of $h$, as in the Neural ODE case. The filter
\eqref{eq:qp-safety} then yields the safe velocity.
Figure ~\ref{fig:main} shows the filtered trajectories on the learned LPV-DS.

\subsection{SE(3) Reachability}
\label{subsec:se3}
Finally, we extend the certificate to a DS on $\mathrm{SE}(3)$ and deploy it on
a Franka Research 3. We build on the stable Lie-group vector field of
\cite{urain2022learningstablevectorfields, lpvdsse3}, which learns a diffeomorphism $\Phi$ onto a latent space
with a simple stable flow, and add the reachability certificate on top.

A pose $H \in \mathrm{SE}(3)$ is expressed relative to the goal $H_o$ and
flattened to a tangent vector $\hat{x} = \mathrm{LogMap}(H_o^{-1} H)$ (zero at
the goal), which $\Phi$ maps to a latent coordinate $z = \Phi(\hat{x})$ with
normalized stable dynamics $\dot{z} = -\alpha z$. Stability in the latent space
is inherited on the manifold \cite{lee_introduction_2018} through $\Phi$.

Because the obstacle is physical, we evaluate safety on the end-effector
position in the task space. Rolling the latent decay forward, $z(\tau_k) = z_0\,e^{-\alpha\tau_k}$
for $\tau_k \in [0, T_{\mathrm{fin}}]$ with
$T_{\mathrm{fin}} = \tfrac{1}{\alpha}\log(\|z_0\|/\eta)$, each sample is
mapped back to a pose
$H(\tau_k) = H_o\,\mathrm{ExpMap}(\Phi^{-1}(z(\tau_k)))$ and its end-effector
position $p(\tau_k)$ read off. The value function is the worst clearance over
the rollout, smoothed by a softmin,
\begin{align}
    V(z_0) = -\frac{1}{\beta} \log \sum_{k=0}^{n_\tau - 1} e^{-\beta\, d_k},
    \qquad d_k = \|p(\tau_k) - c\| - r,
\end{align}
which under-approximates $\min_k d_k$ and is differentiable, giving
$\nabla_{z_0} V$ by autograd.

The filter \eqref{eq:qp-safety} is applied in the latent space and pulled back
to the observation tangent space through the Jacobian
$J = \partial\Phi/\partial\hat{x}$. We obtain the next commanded pose through $\mathrm{ExpMap}$.

 \newcommand{\stackcol}[2]{%
    \begin{minipage}[t]{0.19\textwidth}%
      \centering
      \includegraphics[width=\linewidth]{#1}%
      \par\nointerlineskip          
      \includegraphics[width=\linewidth]{#2}%
    \end{minipage}}

\section{Results}
We evaluate the certificate from \eqref{eq:truncated-value} and \eqref{eq:qp-safety} along two perspectives. First, we demonstrate that the reachability construction transfers unchanged across learned, analytical and latent dynamics. Second, we compare the behavior of our formulation against the standard local methods given by control barrier function and modulation, showing how they differ when they intervene. Finally, we validate the application on SE(3) LieFlow dynamics by applying it to a Franka manipulator and demonstrating its effectiveness on the CLFD dataset. 
\subsection{Generality across constructions}
\label{sec:res-general}

Figure~\ref{fig:main}(b) shows the Neural-ODE construction on a LASA \cite{lasa} shape. The
learned value network gives a $\{V=0\}$ black boundary that the filtered rollout
respects while the unfiltered demonstration-following flow would enter it. Figure \ref{fig:main}(c) shows the effect of the method on a DS learnt from J-shape using LPV-DS. 
Figure~\ref{fig:main}(d) shows the diffeomorphic latent construction. The
certificate is enforced in the latent space, where the DS is a unit-speed radial
decay and $V$ is the closed-form signed distance from the latent line to the
pushed-forward obstacle, and the safe velocity is mapped back through
$\psi^{-1}$ to the observation space where the perturbed trajectory clears the real obstacle and
converges. 

\begin{figure}[!t]
  \centering
  \includegraphics[width=\columnwidth]{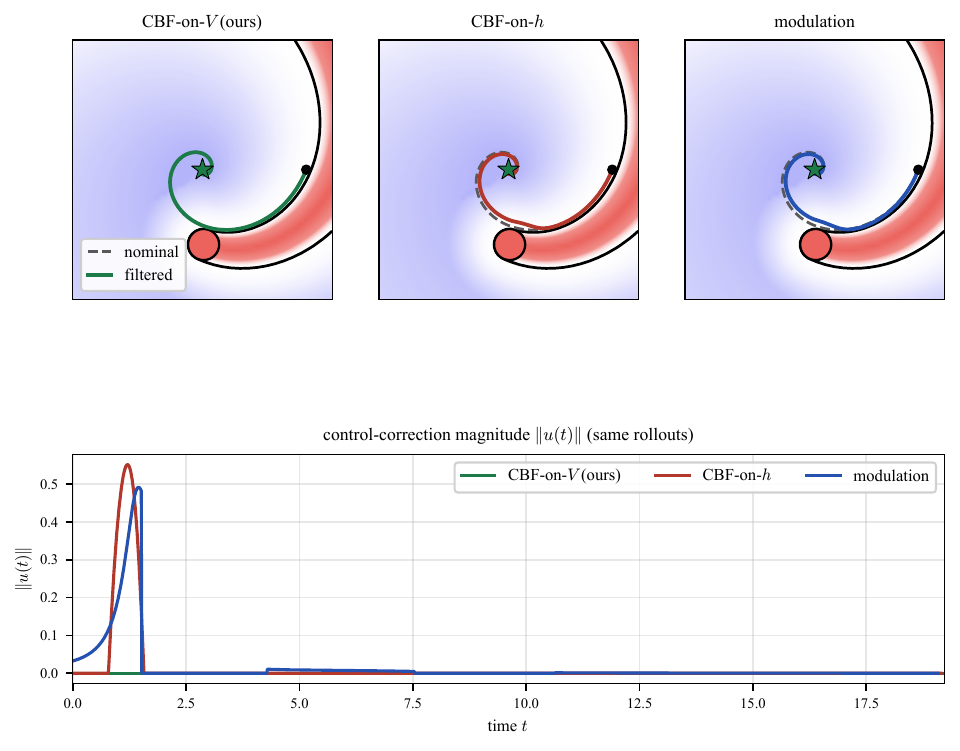}
  \caption{Three filters on the spiral DS from the same initial condition over
    the same obstacle acting on a nominally safe trajectory that grazes by the obstacle. CBF-on-$V$ (ours) coincides with the nominal, while CBF-on-$h$
    and modulation deflect a trajectory that was already safe. Bottom: the three
    correction magnitudes $\lVert u(t)\rVert$ for the rollout sequence.}
  \label{fig:cbf}
\end{figure}
\begin{table}[t]
\centering
\setlength{\tabcolsep}{4pt}         
\caption{Filter behaviour on the spiral DS. They
differ in intervention. CBF-on-$V$ never activates as nominal trajectory is already safe. Metrics: intervention rate
$=\#\{\|u\|>\epsilon\}/N$; total effort $=\int\|u\|\,dt$; peak
$=\max_t\|u\|$; path deviation $=\int\|x_{\mathrm{filt}}-x_{\mathrm{nom}}\|\,dt$;}
\label{tab:cbf}
\begin{tabular}{lccc}
\toprule
Metric & CBF-on-$V$ (ours) & CBF-on-$h$ & mod. \\
\midrule
Intervention rate
& \multirow{4}{*}{\shortstack{$\mathbf{0}$\\[2pt]\scriptsize(never\\\scriptsize activates)}}
& 0.041 & 0.278 \\
Total effort     &  & 0.288 & 0.300 \\
Peak correction  &  & 0.552 & 0.492 \\
Path deviation   &  & 0.692 & 0.365 \\
\midrule
\bottomrule
\end{tabular}
\end{table}

\subsection{Minimum intervention versus local safeguards}
\label{sec:res-cbf}

The comparison in this section uses the spiral DS, a single initial condition, and a
single obstacle, and runs three filters on the system
(Figure ~\ref{fig:cbf}, Table~\ref{tab:cbf}):
(i)~\emph{CBF-on-$V$} (ours), the barrier of Eq.~\eqref{eq:qp-safety} on the
reachability value;
(ii)~\emph{CBF-on-$h$}, the same barrier on the instantaneous margin to obstacle
$h(x)=\lVert x-x_{\mathrm{obs}}\rVert-r$; and
(iii)~\emph{modulation}~\cite{onmanifold}, the closed-form deflection
$\dot x=M(x)f(x)$ with reference point at the obstacle centre.
The obstacle is placed just off the spiral arc so the nominal trajectory
\emph{grazes} it but doesn't collide with it.

Intervention is the main comparison characteristic. Because $V$ has a look-ahead component and is non-negative everywhere, the
constraint of Eq.~\eqref{eq:qp-safety} is inactive at every step i.e., CBF-on-$V$
\emph{never activates}. CBF-on-$h$ reads only the local margin value and
therefore fires a brief perturbation as the trajectory passes the obstacle. Modulation deflects whenever the flow points
toward the obstacle and so engages from afar. The bottom panel of Figure~\ref{fig:cbf} shows
the three $\lVert u(t)\rVert$ traces on a shared axis. Note that ours is flat at zero, the
barrier on $h$ is an impulsive spike close to the obstacle, and modulation is a more sustained correction.

\begin{figure*}[t!]
  \centering
  \begin{minipage}[t]{0.45\textwidth}\centering
    \includegraphics[width=0.9\linewidth]{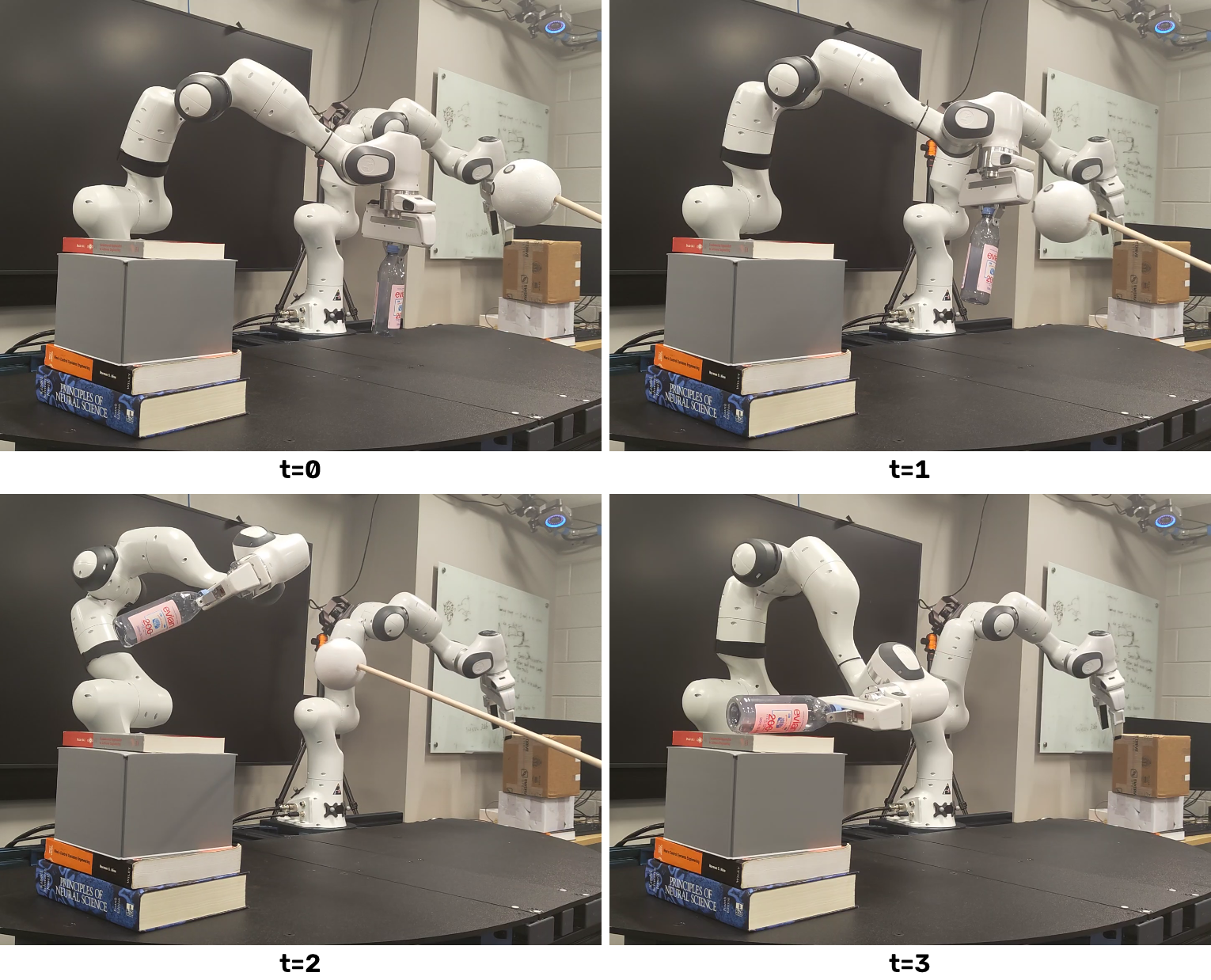}\\[2pt]
    \includegraphics[width=\linewidth]{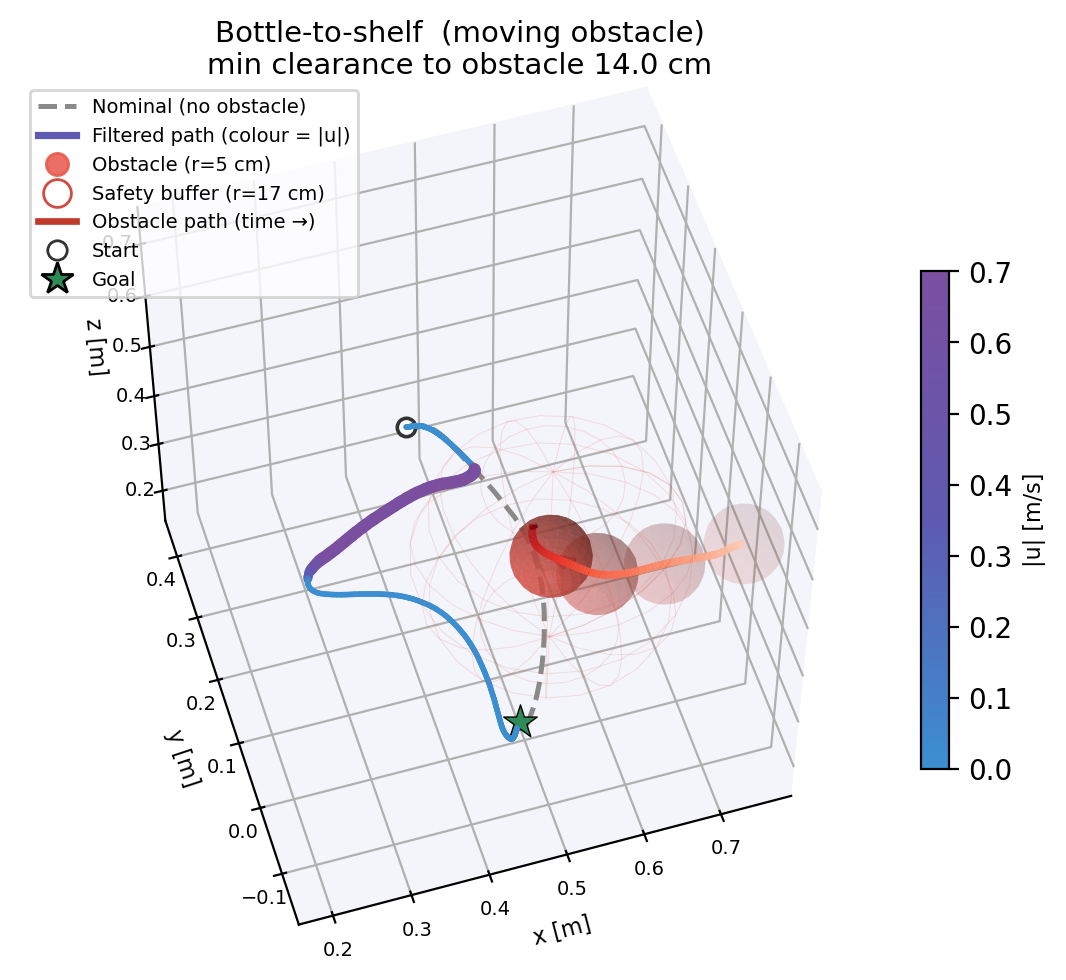}\\[3pt]
    {\small (a) Bottle-to-shelf}
  \end{minipage}\begin{minipage}[t]{0.45\textwidth}\centering
    \includegraphics[width=0.9\linewidth]{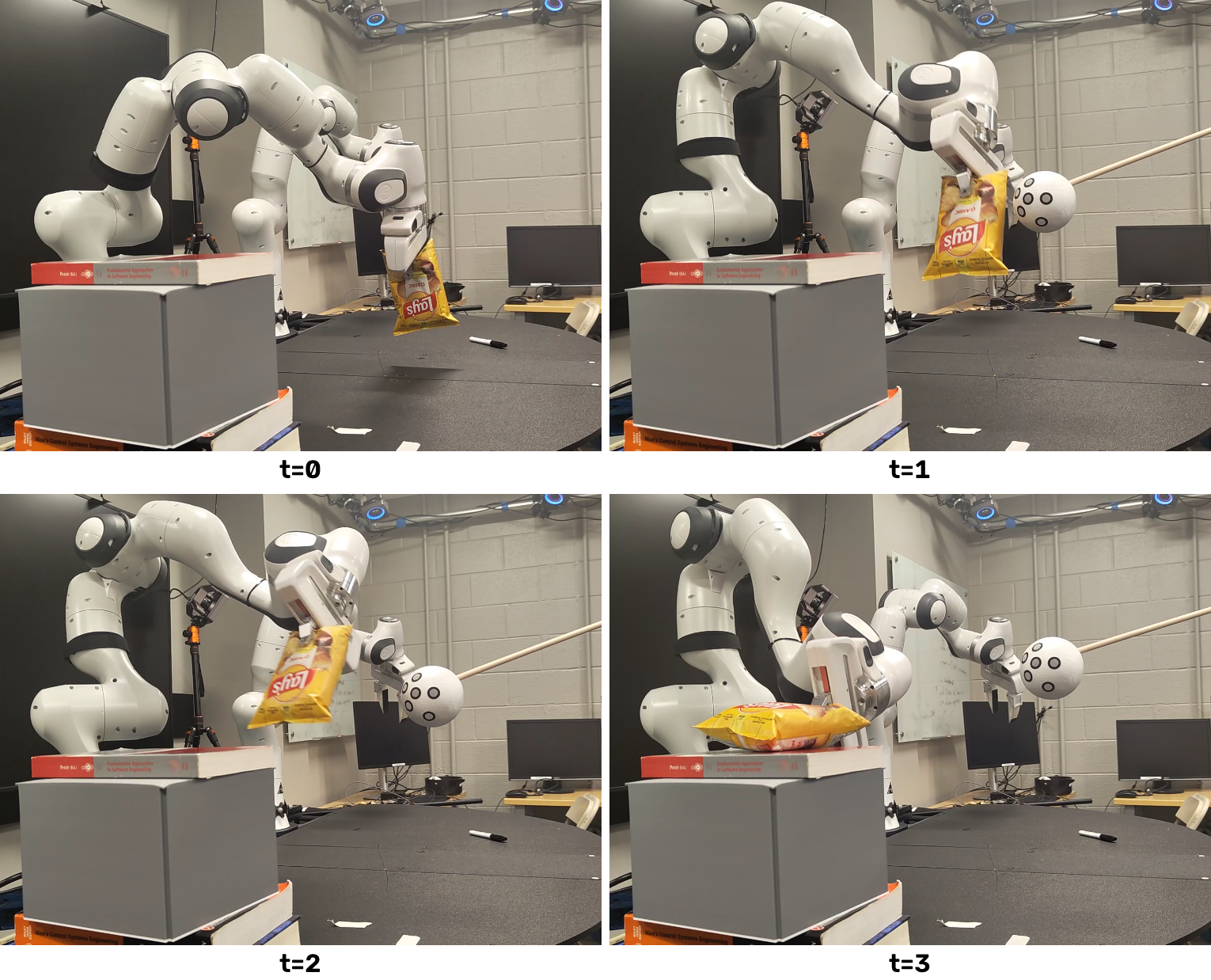}\\[2pt]
    \includegraphics[width=\linewidth]{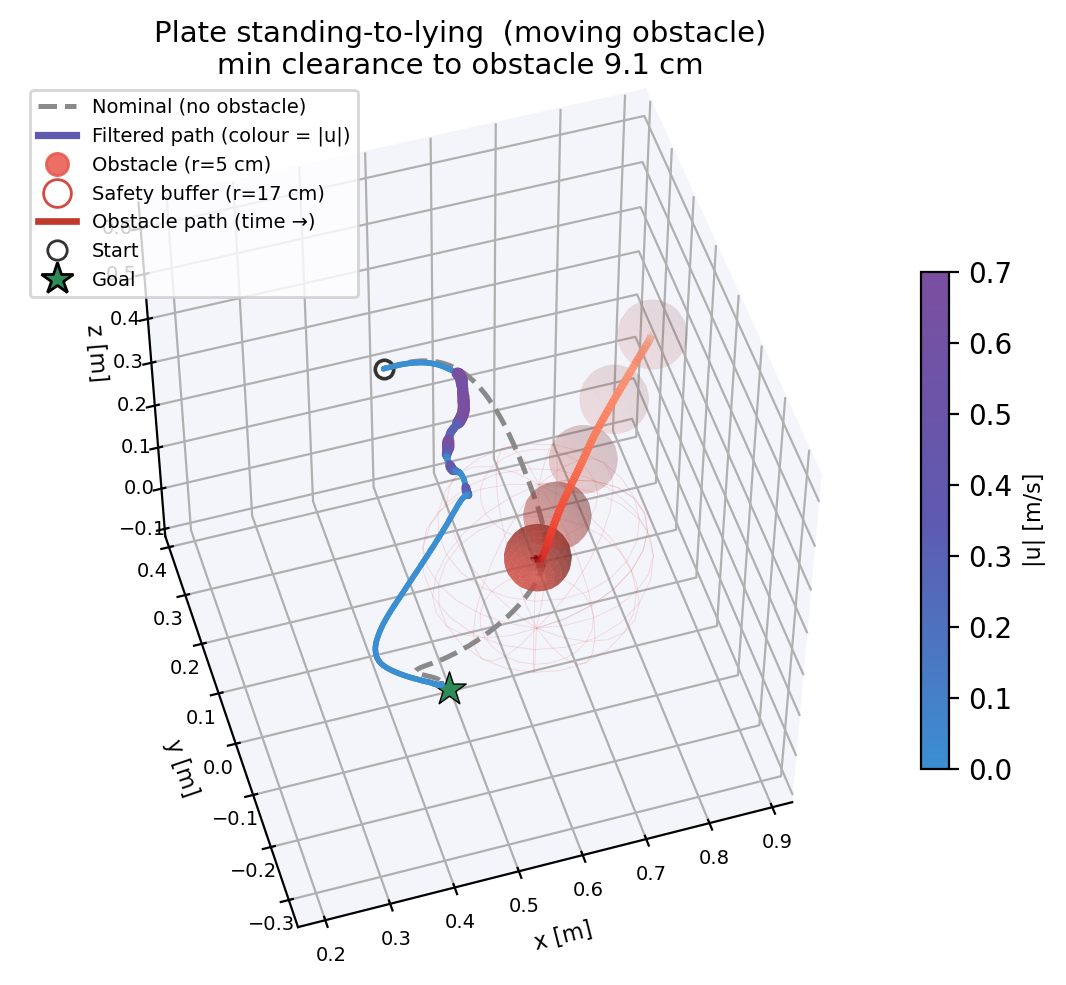}\\[3pt]
    {\small (b) Plate standing-to-lying}
  \end{minipage}
  \caption{Reactive avoidance of a hand-moved obstacle on the Franka for two
  SE(3) skills. Top: video frames at $t=0\text{--}3$. Bottom: nominal (grey
  dashed) and filtered end-effector paths, the latter coloured by the task-space
  correction $\lVert u\rVert$; the mocap-tracked obstacle is a time-graded red
  trail, solid at closest approach.}
  \label{fig:reactive}
\end{figure*}

\section{SE(3) LieFlow on the Franka Arm}
We deploy the SE(3) certificate of Section~\ref{subsec:se3} on a Franka
Research~3 to test whether the same reachability filter transfers, unchanged,
from simulation to a physical 7-DoF manipulator executing full-pose skills. We
use three skills from the CLFD RobotTasks benchmark~\cite{clfd}---\emph{pouring}, \emph{bottle-to-shelf}, and
\emph{plate-standing-to-lying}---each a demonstrated SE(3) trajectory in which
end-effector position and orientation are reproduced jointly. Every skill is
learned as a stable SE(3) LieFlow DS~\cite{urain2022learningstablevectorfields,lpvdsse3} trained only to imitate
the demonstrations; obstacle avoidance is added at deployment by the filter,
with no retraining of the policy.

\subsection{Deployment}
The filtered velocity $\dot{\hat{x}}$ tracked by a passive,
energy-bounded impedance controller at $100$~Hz \cite{7358081}.  The end-effector frame
velocity $\dot p_{ee}=v_s + w_s\times p_{ee}$ drives the linear channel and
$w_s$ the wrist, together reproducing the rigid-body motion of the learned
field. A spherical obstacle of physical radius $5$~cm, inflated by a safety
buffer to a radius $r_\text{buf}$, is tracked by motion capture and its centre
$c$ is streamed live, so the failure region is evaluated on the measured
end-effector position in task space. At every control tick we evaluate the
softmin value $V$ and its gradient $\nabla V$ on a GPU at the current pose and
obstacle centre and apply the safety projection~\eqref{eq:qp-safety} before
mapping the chart velocity back to a twist.

We report two evaluations that probe
complementary properties of the certificate: (i) its real-time response to a
\emph{moving} obstacle, and (ii) its intervention behaviour against a local baseline
under a \emph{fixed} obstacle.

\subsection{Anticipatory versus reactive intervention}
With the obstacle fixed, we contrast our reachability filter (CBF-on-$V$)
against a geometric distance CBF $h=\lVert p_{ee}-c\rVert-r_\text{buf}$
enforced by a min-norm QP on the velocity command, both acting on the same learned
policy. For each task
Table~\ref{tab:se3} reports the intervention lead, the active fraction,  the boundary distance (arc length travelled adjacent to the
obstacle surface), and the RMS command jerk. None of the runs collide. Because the reachability value looks
ahead, our filter begins correcting $0.9$--$2.1$~s before the closest
approach which is much earlier than the distance barrier,
which stays dormant until the buffer is reached and engages only $\sim\!0.1$~s
ahead. Correcting early, it skirts the obstacle over a shorter arc on every task and commands consistently lower jerk. This behavior is expected from the CBF-on-reachability approach.

In sum, the distance barrier reacts late at the buffer and
brakes hard, whereas the reachability value anticipates the learned flow and
corrects early and lightly.

\begin{table}[t]
\centering
\caption{SE(3) hardware: reachability filter (CBF-on-$V$, ours) vs.\ geometric
distance CBF (CBF-on-$h$) on the same learned policy per task; no collision in
any run. Arrows give the better direction; best per
task/metric in bold.}

\label{tab:se3}
\setlength{\tabcolsep}{4pt}
\footnotesize
\resizebox{\columnwidth}{!}{%
\begin{tabular}{lcccccc}
\toprule
& \multicolumn{2}{c}{Bottle-to-shelf} & \multicolumn{2}{c}{Pouring} & \multicolumn{2}{c}{Plate std.-to-lying} \\
\cmidrule(lr){2-3}\cmidrule(lr){4-5}\cmidrule(lr){6-7}
Metric & ours & CBF-$h$ & ours & CBF-$h$ & ours & CBF-$h$ \\
\midrule
Intervention lead (s)\,$\uparrow$          & \textbf{2.12}  & 0.11  & \textbf{1.11}  & 0.12  & \textbf{0.86}  & 0.12 \\
Correction-active fraction\,$\downarrow$   & \textbf{0.067} & 0.095 & \textbf{0.239} & 0.322 & 0.205 & \textbf{0.139} \\
Boundary distance (cm)\,$\downarrow$       & \textbf{13.6}  & 34.8  & \textbf{34.4}  & 42.4  & \textbf{20.9}  & 33.3 \\
Command jerk RMS (m/s$^3$)\,$\downarrow$   & \textbf{163}   & 460   & \textbf{281}   & 1104  & \textbf{85}    & 92 \\
\bottomrule
\end{tabular}
}
\end{table}

\subsection{Reactive avoidance under a moving obstacle}
In this task we use motion capture to track a hand-held sphere that a person
moves into the end-effector path during execution, treating it as a dynamic
obstacle: its centre $c$ is streamed live at every control tick, so the failure
region is non-stationary. The filter re-evaluates $V$ and $\nabla V$ at the
current pose and $c$ each tick and re-plans the avoidance online.

Figure~\ref{fig:reactive} shows the filtered end-effector path for two skills,
coloured by the task-space correction $\lVert u\rVert$, with the obstacle's
mocap trajectory drawn as a time-graded trail that is solid at closest approach.
The correction stays near zero over most of the path and rises only over a short
arc as the sphere enters, steering the arm around it; the arm keeps tracking the
demonstrated skill and reaches the goal pose in both position and orientation
without contacting the obstacle, even as the obstacle keeps moving.

Because $V$ scores the worst-case clearance along the predicted closed-loop rollout rather than the instantaneous distance to the obstacle, the filter anticipates the incoming sphere and corrects early. This look-ahead, minimum-intervention behaviour
is the practical advantage of the reachability certificate: it delivers
real-time, anticipatory safety that transfers unchanged to a physical SE(3) system and holds even when the failure set is non-stationary.

\section{Conclusion}

This paper explores the application of the concept of reachability in the context of dynamical systems yields valuable theoretical and behavioral results such as a description of the maximal invariance set and the early correction behavior. It also lends to an algorithm that runs on high-dimensional systems in the SE(3) manifold, verifying its practicality. The framework is general in that it can be applied to multiple dynamical systems settings with relative ease. For future work, the extension of this application to multiple objects is a non-trivial work that we plan to undertake.

\bibliographystyle{IEEEtran}
\bibliography{refs}
\end{document}